\documentclass[tablecaption=bottom]{jmlr}
 
\usepackage{thmtools} 
\usepackage{thm-restate}
\newcommand{\R}{\mathbb{R}}
\usepackage{cleveref}
\crefname{prop}{proposition}{propositions}
\newcommand{\KL}{\mathrm{D}_{KL}}
\newtheorem{prop}{Proposition}
\newtheorem{lem}{Lemma}
\newtheorem{defi}{Definition}

\newcommand{\datafeatures}{\phi_X}
\newcommand{\Afeatures}{\phi_A}
\newcommand{\transpose}{^\mathrm{\textsf{\tiny T}}}
\newcommand{\Tr}{\textup{tr}}
\newcommand{\mcN}{\mathcal{N}}
\newcommand{\variationalfamily}{\mathcal{Q}}
 \newcommand{\rationals}{\mathbb{Q}}
 \newcommand{\naturals}{\mathbb{N}}
\usepackage{mathtools}

\newcommand{\algname}[1]{\textsc{#1}} 
\newcommand{\dsetname}[1]{\textsc{#1}}  
\theorembodyfont{\upshape}
\theoremheaderfont{\scshape}
\theorempostheader{:}
\theoremsep{\newline}

\title[Understanding VI in Function-Space]{Understanding Variational Inference in Function-Space}

\author{\Name{David R. Burt} \Email{drb62@cam.ac.uk}\\
\Name{Sebastian W. Ober} \Email{swo25@cam.ac.uk}\\
\Name{Adri\`{a} Garriga-Alonso} \Email{ag919@cam.ac.uk}\\
\addr Department of Engineering, University of Cambridge, UK
\AND
\Name{Mark {van der Wilk}} \Email{m.vdwilk@imperial.ac.uk}\\
\addr Department of Computing, Imperial College London, UK
}

\begin{document}
\maketitle
\thispagestyle{plain}

\begin{abstract}
Recent work has attempted to directly approximate the `function-space' or predictive posterior distribution of Bayesian models, without approximating the posterior distribution over the parameters. This is appealing in e.g.~Bayesian neural networks, where we only need the former, and the latter is hard to represent. In this work, we highlight some advantages and limitations of employing the Kullback-Leibler divergence in this setting. For example, we show that minimizing the KL divergence between a wide class of parametric distributions and the posterior induced by a (non-degenerate) Gaussian process prior leads to an ill-defined objective function. Then, we propose (featurized) Bayesian linear regression as a benchmark for `function-space' inference methods that directly measures approximation quality. We apply this methodology to assess aspects of the objective function and inference scheme considered in \citet*{sun2018functional}, emphasizing the quality of approximation to Bayesian inference as opposed to predictive performance.
\end{abstract}

\section{Introduction}\label{sec:intro}
While neural networks offer a successful parametric representation of functions, performing Bayesian inference over the parameters is challenging. Since only the predictions matter, a recent line of work has focused on directly approximating the posterior predictive distribution in `function-space' \citep{sun2018functional,ma2019variational,wang2019function}, with the intent of reducing the influence of the specific parameterization on the quality of inference. Similar to approximate Gaussian processes, a variational approach can be taken, where a discrepancy (e.g.~a KL divergence) is minimized between approximate and posterior predictive distributions \citep{alexander2016matthews}. However, we are dealing with measures on functions, so care is needed to ensure that the divergence is well-defined and a useful objective function.

In this work, we investigate situations where this may not form a useful objective.  \Citet{sun2018functional} hint that: ``the function space KL divergence may be infinite, for instance if the prior assigns measure zero to the set of functions representable by a neural network''. We give examples of ill-defined objective functions that arise in the context of variational inference between approximate and exact predictive posteriors. We give a proof that variational inference using the KL divergence (or any $f$-divergence) does not lead to a sensible objective function when the prior is a non-degenerate Gaussian process and the variational family contains only `nice' parametric models (or vice versa). We show similar results for single hidden layer (1HL) Bayesian neural networks (BNNs) in the case when the prior has a different width than the approximate posterior and ReLU activation functions are used. Our proofs are contained in the appendix.

Having established that issues with the objective function can arise; we consider how to assess the quality of the approximation to the posterior achieved by methods motivated by performing variational inference in function-space. We propose Bayesian linear regression (BLR) as a benchmark for measuring the effectiveness of functional variational inference schemes. Since the exact posterior, as well as the KL divergence in function-space, can be computed in this case, we can directly assess the quality of \emph{approximate inference} separately from the quality of predictive performance, which may be good even in cases where inference is not accurate. We show how this benchmark can be used to decouple the impacts of further approximations to the objective function made by \citet{sun2018functional}. This approach gives a principled starting point for assessing future improvements in functional variational inference.

\section{Background}\label{sec:background}
We begin by introducing some measure theory notation, which is required to handle the distributions on functions that we perform inference with. We review the data processing inequality, which is needed for the proofs, and discuss the work of \citet{sun2018functional} which we analyze further in this work.

\paragraph{Notation}
A probability measure, $P$, is a function from subsets (`events') to $[0,1]$, such that the probability of the event $E$ is given by $P(E)$. In the case when the subsets are contained in $\R^k$, this can sometimes be related to a (Lebesgue) density $p:\R^k \to [0, \infty)$ by $P(E) = \int_{z\in E} p(z)dz$. A random variable, $Z$, is said to have distribution $P$ (and we write $Z \sim P$) if for all events $E$, $\mathrm{Prob}(Z \in E)=P(E)$. For measures defined on the same event space, we write $Q \ll P$ if for all $E$ such that $P(E)=0$, we have $Q(E)=0$. The \emph{Kullback-Leibler (KL) divergence} between two probability measures $Q$ and $P$ is given by, 
\begin{equation}
    \mathrm{D}_{KL}(Q, P) = \begin{cases}
                                \int \log \frac{dQ}{dP}dQ & Q \ll P, \\
                                \infty & \text{otherwise,}
                            \end{cases}
\end{equation}
where $\frac{dQ}{dP}$ denotes the Radon-Nikodym derivative, which is simply the ratio of the densities of these measures when the densities exist.\footnote{As we are often interested in the case when $Q$ and $P$ are defined on spaces without a Lebesgue measure, we use this more general formulation of KL divergence.} While we focus on the KL divergence in the main text due to its wide-spread use in variational inference, our results can be naturally extended to other $f$-divergences as described in \cref{app:f-divs}.


\paragraph{The data processing inequality}
Given a random variable $Z \sim P$, we can transform it by a function $g$ to find a new random variable $g(Z)$. We refer to its distribution as the \emph{pushforward measure} of $P$, which we denote $g_* P$. The \emph{data processing inequality} states that if two random variables are transformed in this way, they cannot become easier to tell apart. 
 
\begin{prop}[Data processing inequality~{\citealp[Thm 6.2]{polyanskiy2014lecture}}]\label{prop:dpi}
  \hspace{-.3cm} Let $g$ be a measurable function, then $\mathrm{D}_{KL}\left(g_*P, g_*Q\right) \leq \mathrm{D}_{KL}\left(P, Q\right)$,
\end{prop}

\paragraph{Background on Functional Variational Inference}

We consider the application of variational inference to regression. In particular, we assume data $D=\{(x_n, y_n)\}_{n=1}^N$ has been observed, with $x_n \in \mathcal{X}=\R^d$ and $y_n \in \R$. We assume an additive noise model, i.e~$\hat{y}(x_*) = \hat{f}(x_*) + \epsilon_*$, where $\hat{f}$ is a stochastic process indexed by $\mathcal{X}$ and each $\epsilon_*$ is an independent mean-zero random variable. We assume a priori that $\hat{f} \sim P$. The goal is to approximate the \emph{posterior} distribution of $\hat{f}$ given the data $D$, $P_D$. Define $\ell_D: \R^\mathcal{X} \to \R$ to be the likelihood function given the observed data $D$.\footnote{Commonly this is written $\prod_{n=1}^N p(y_n|x_n,\hat{f})$.} Given an approximate distribution $Q$, variational inference \citep{blei2017variational} gives us the evidence lower bound (ELBO)
\begin{equation}\label{eqn:elbo}
    \log \int \ell_D dP  \geq \int \log \ell_D dQ - \mathrm{D}_{KL}(Q, P).
\end{equation}
 Maximizing the RHS of \cref{eqn:elbo} over $Q \in \variationalfamily$ is equivalent to minimizing $\mathrm{D}_{KL}(Q, P_D)$. Moreover, the RHS of \cref{eqn:elbo} can often be estimated: $\int \log \ell_D dQ$ can be estimated with Monte Carlo methods so long as we can evaluate $\ell_D$ and sample from $Q$. In the context of variational inference in parameter space, $Q$ is generally constrained to be from some family such that $\mathrm{D}_{KL}(Q, P)$ is tractable (e.g.~if $P$ is an isotropic Gaussian distribution, then $Q$ is often chosen to be Gaussian so that this KL divergence can be evaluated in closed form).

\Citet{sun2018functional} proposed using \cref{eqn:elbo} with $Q$ and $P$ the approximate predictive and prior predictive distributions, in which case $D_{KL}(Q, P)$ is a divergence between measures on the infinite product space $\R^\mathcal{X}$. The starting point for their work is:
\begin{prop}[{\citealt[Theorem 1]{sun2018functional}}]
For measures $Q,P$ on (the product $\sigma$-algebra of) $\R^\mathcal{X}$, 
\begin{equation}\label{eqn:fkl}
    \mathrm{D}_{KL}(Q, P) = \sup_{\mathrm{X}\subset \mathcal{X},\, |X|<\infty} \mathrm{D}_{KL}(Q_X, P_{X}),
\end{equation}
where $Q_X$, $P_X$ are the marginals of the measures $Q$ and $P$ on the set $X$.\footnote{Precisely, $Q_X = \pi_{X*}Q, P_X = \pi_{X*}P$, where $\pi_X$ denotes the canonical projection (\cref{def:canonical-pi}) onto $x \in X$.}
\end{prop}

In other words, the KL divergence between the stochastic processes is equal to the supremum of the KL divergence between the measures restricted to finite marginals. Substituting \cref{eqn:fkl} into \cref{eqn:elbo} introduces two sources of intractability. First, the supremum is over uncountably many subsets, and will be generally intractable. Second, the distributions $Q_X,P_X$ are often defined implicitly through a tractable sampling procedure, which does not provide closed form densities (with the notable exception when either $Q_X$ or $P_X$ is a Gaussian measure). 

\Citet{sun2018functional} propose replacing the supremum with an expectation, and using this in \cref{eqn:elbo} to address the first intractability. This involves defining a distribution over finite subsets of $\mathcal{X}$, e.g.~sampling points from the data as well as uniformly from a subset of $\mathcal{X}$, and using the KL divergence between the approximate posterior and the prior restricted to this index set.  The second intractability can be addressed using any form of implicit inference \citep{huszar2017variational}. \Citet{sun2018functional} use the spectral Stein gradient estimator (SSGE; \citet{shi2018spectral}) to obtain estimates of the gradient of the KL.

\section{Properties of the KL divergence in function-space}\label{sec:theorems}

In this section we discuss properties of KL divergences in function-space, noting that our results generalize with minor modifications to all $f$-divergences (\cref{app:f-divs}). We focus on finding conditions under which \cref{eqn:elbo} is a well-defined objective. We discuss the case when parametric models are used to describe both the prior and the approximate posterior, and then move to the case when a Gaussian process is used as either the approximate posterior or the prior. 

\paragraph{Parametric distributions}

We call a distribution parametric if it is described by a probability distribution over a parameter space, which we assume is $\R^k$, as well as a mapping from parameters $\Theta$ to functions, i.e.~$\hat{f}(x)= h(x,\Theta)$ for all $x \in \mathcal{X}=\R^d$. Note that $\hat{f}$ is a random function (stochastic process) indexed by $\mathcal{X}$. Moreover, from its definition, we see that given $\Theta \sim P_\Theta$, upon  defining $g(\Theta)(x)=h(x,\Theta)$, we have $\hat{f} = g(\Theta) \sim g_* P_\Theta$. We call such a distribution parameterized by the pair $(P_\Theta, g)$. The function $g: \R^k \to \R^\mathcal{X}$ is the mapping from parameters to functions and is assumed to be measurable. From the data processing inequality (\cref{prop:dpi}), we make the following observation:
\begin{prop}\label{prop:klinequality}
Suppose that the approximate posterior is parameterized by $(Q_\theta, g)$ and that the prior is parameterized by $(P_\theta,g)$. Define $Q=g_*Q_\theta$ and $P=g_*P_\theta$ to be the approximate posterior predictive and the prior predictive respectively. Then,
\begin{equation}\label{eqn:dpi-fvi}
    \mathrm{D}_{KL}(Q, P) \leq \mathrm{D}_{KL}(Q_\theta, P_\theta).
\end{equation}
If $g$ is injective (each set of parameters leads to a unique predictive function) equality holds. 
\end{prop}
\begin{remark}
\Cref{prop:klinequality} implies that in cases where we can perform variational inference in parameter space, variational inference in function-space is also well-defined. While the ELBO using KL-divergence in parameter-space depends on the choice of parameterization (unless $g$ is injective), the ELBO in function-space does not.
\end{remark}
\Citet{ma2019variational} observed the inequality in \cref{eqn:dpi-fvi} in the context of performing inference with stochastic processes; it is an immediate consequence of the data processing inequality. The equality can be derived by noting that if $g$ is injective, it has a left inverse and the data-processing inequality can be applied in the opposite direction. We note that it is possible that for specific $Q, P$ equality can hold even if $g$ is not injective: in other words, the converse is not generally true. 

When applied to variational inference in function-space, this inequality tells us that the KL divergence in function-space is no larger than the KL divergence in parameter space, implying that the evidence lower bound (ELBO) obtained in function-space must be closer to the log marginal likelihood than the ELBO in parameter space. 

\paragraph{Bayesian neural networks}
The above discussion applies immediately to variational inference with Bayesian neural networks (after establishing measurability of $g$), given the approximate posterior and the prior have the same architecture. This can be seen as motivation for using functional variational inference, particularly since BNNs are highly non-identifiable (e.g.~permuting neurons leads to identical predictions), and equality between parameter-space and function-space KLs does not generally hold in these models. A natural question is whether variational inference can be applied when the prior and candidate approximate posteriors have different architectures. We show in \cref{app:bnn-proof} that in the case of single hidden ReLU networks, if the prior and approximate posterior have different widths, and both $Q_\Theta$ and $P_\Theta$ have densities (with respect to Lebesgue measure), then the KL divergence between the approximate posterior and the prior is infinite. We conjecture that this result is true much more broadly for Bayesian neural networks of different architectures, excluding cases where architectures are trivially the same (e.g.~if a neuron is added but the outgoing weight is $0$ with probability $1$ so that the additional neuron is always pruned).

\begin{remark}
\Cref{prop:bnn-diff-widths} (\cref{app:bnn-proof}) shows that variational inference in function space is not always well-defined when both the prior and approximate posteriors are defined using neural networks, if the architectures are not the same.
\end{remark}

\paragraph{Parametric distributions and Gaussian processes}
A Gaussian process (GP) is a random function such that when the function is indexed at any finite collection of points, the distribution of the function values is multivariate Gaussian. We call a Gaussian process non-degenerate if there exist arbitrarily large collections of points where, when we evaluate the function at these points, the resulting multivariate Gaussian has a full-rank covariance matrix (equivalently has a density with respect to the appropriate Lebesgue measure).\footnote{To be non-degenerate, the GP needs infinite basis functions \citep[\S 4.3]{gpml}.}

Gaussian processes have been proposed for use in functional inference schemes both as priors \citep{sun2018functional} and as approximate posteriors \citep{ma2019variational}. However, under quite general conditions, we show that the KL divergence between Gaussian processes and parametric models is not a useful objective.
\begin{restatable}[]{prop}{parGPs}\label{prop:param-gp-kl}
Let $(Q_\theta,g)$ parameterize the approximate posterior of a parametric model and let $P$ be a non-degenerate Gaussian process. Assume that $g(\cdot)(x)$ is locally Lipschitz for all $x$. Let $Q=g_*Q_\theta$. Then, $\KL(Q,P) = \infty$ and $\KL(P,Q) = \infty$.
\end{restatable}

The assumption that $g$ is locally Lipschitz in each output is very weak; it holds for most commonly used Bayesian machine learning models, including (deep) BNNs with ReLU, tanh or sigmoid non-linearities. 
\begin{remark}
\Cref{prop:param-gp-kl} tells us that using KL divergences as an objective function to approximate Gaussian processes with parametric models does not lead to a useful objective. However, this does not mean that parametric models cannot approximate Gaussian processes well. This is evidently false from the success of methods such as Random Fourier features \citep{rahimi2007random} and Subset of Regressors \citep[\S 7]{wahba1990spline}.
\end{remark}

While \cref{prop:bnn-diff-widths,prop:param-gp-kl} highlight limitations of variational objective defined in function space, we believe that the overall idea of approximating the predictive posterior as opposed to the parameter-space posterior is well-motivated. \Cref{prop:bnn-diff-widths,prop:param-gp-kl} suggest the need for other objective functions for this task, as well as carefully assessing whether the predictive posterior obtained by variational inference in function space resembles the exact predictive posterior.

\section{Benchmarking Functional Approaches to Variational Inference}\label{sec:benchmarking}

Bayesian linear regression with Gaussian priors can be used as a tool for assessing the quality of variational inference in function-space. We consider the model,
\begin{equation}\label{eqn:blr}
    \hat{y}(x_i) = \Theta\transpose \phi(x_i) + \epsilon_i, \quad \epsilon_i \sim \mathcal{N}(0, \sigma^2), \quad \Theta \sim \mathcal{N}(0, I),
\end{equation}
where $\phi: \mathcal{X} \to \R^k$ is a feature mapping, and $\Theta$ is the (random) vector of weights. Recall the notation introduced in Section~\ref{sec:theorems}, $g(\Theta)(x) = \Theta^T \phi(x)$. In this case, we can verify that $g$ is injective for a given set of features by finding a set of $k$ inputs $A=\{a_1,\dotsc, a_k\}$ such that the vectors $\{\phi(a_1),\dotsc, \phi(a_k)\}$ are linearly independent. Therefore, \cref{prop:klinequality} implies that the KL divergence in parameter and function-space are \emph{exactly the same}. We can therefore expect that successful inference methods with identical variational families should obtain the same approximate posteriors, regardless of whether they are represented in the parameter space or function-space.

Since the exact solution for BLR is Gaussian and can be computed analytically, we can compare different function-space inference methods  by seeing which method finds the solution with the smallest KL divergence to the exact posterior. This allows us to assess the quality of \emph{inference} while avoiding potential issues of model mis-specification, whereby it is possible to achieve good test performance with a poor model by using poor inference.

\begin{figure}[t]
\floatconts
  {fig:toy}
  {\caption{Predictive posteriors for each method using all (full-covariance) Gaussian distributions as $\variationalfamily$, in a toy 1d regression.}}
  {%
    \subfigure[\algname{Exact}]{
      \includegraphics[width=0.23\linewidth]{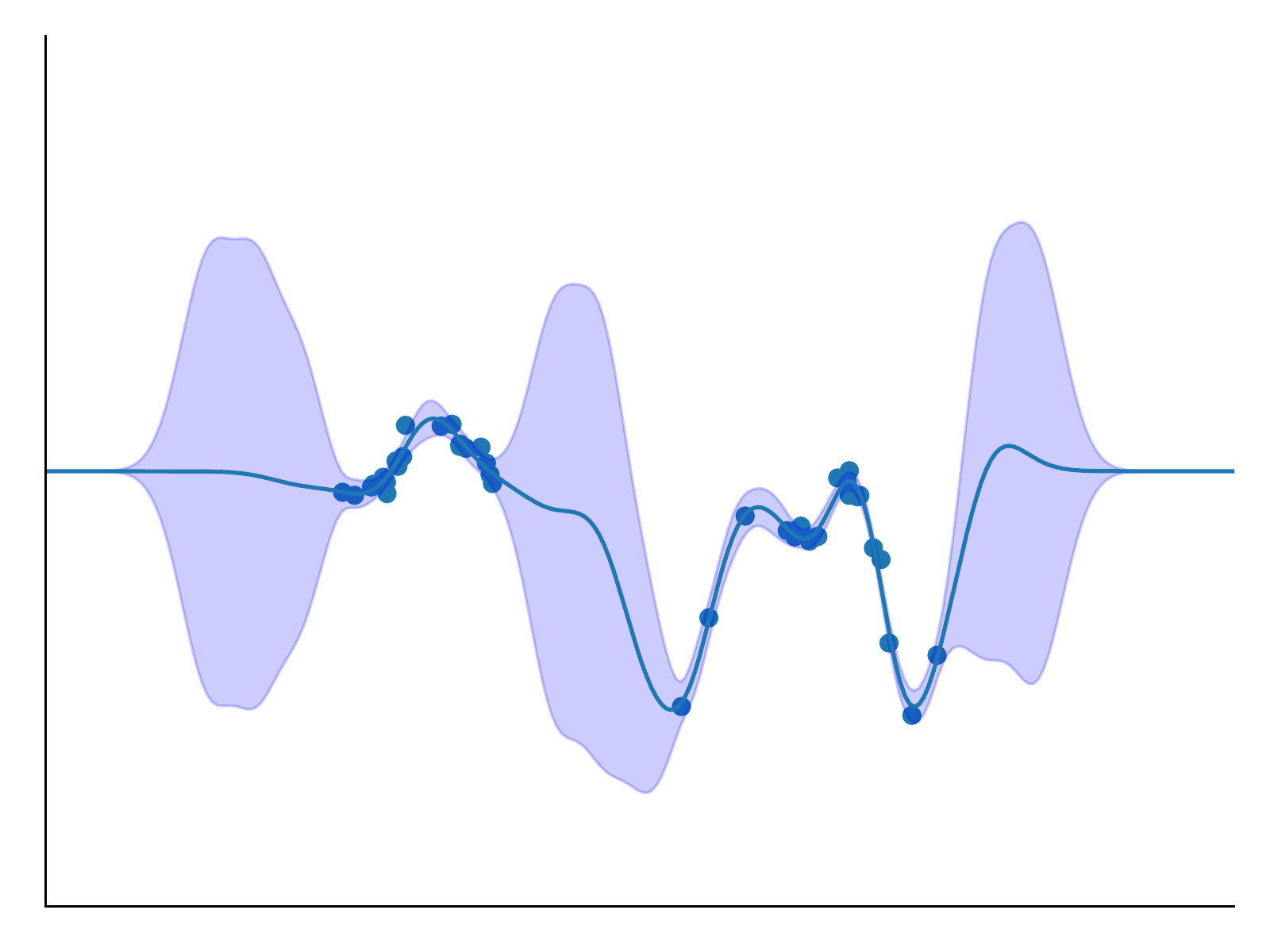}}%
    \subfigure[\algname{FixedA}]{
      \includegraphics[width=0.23\linewidth]{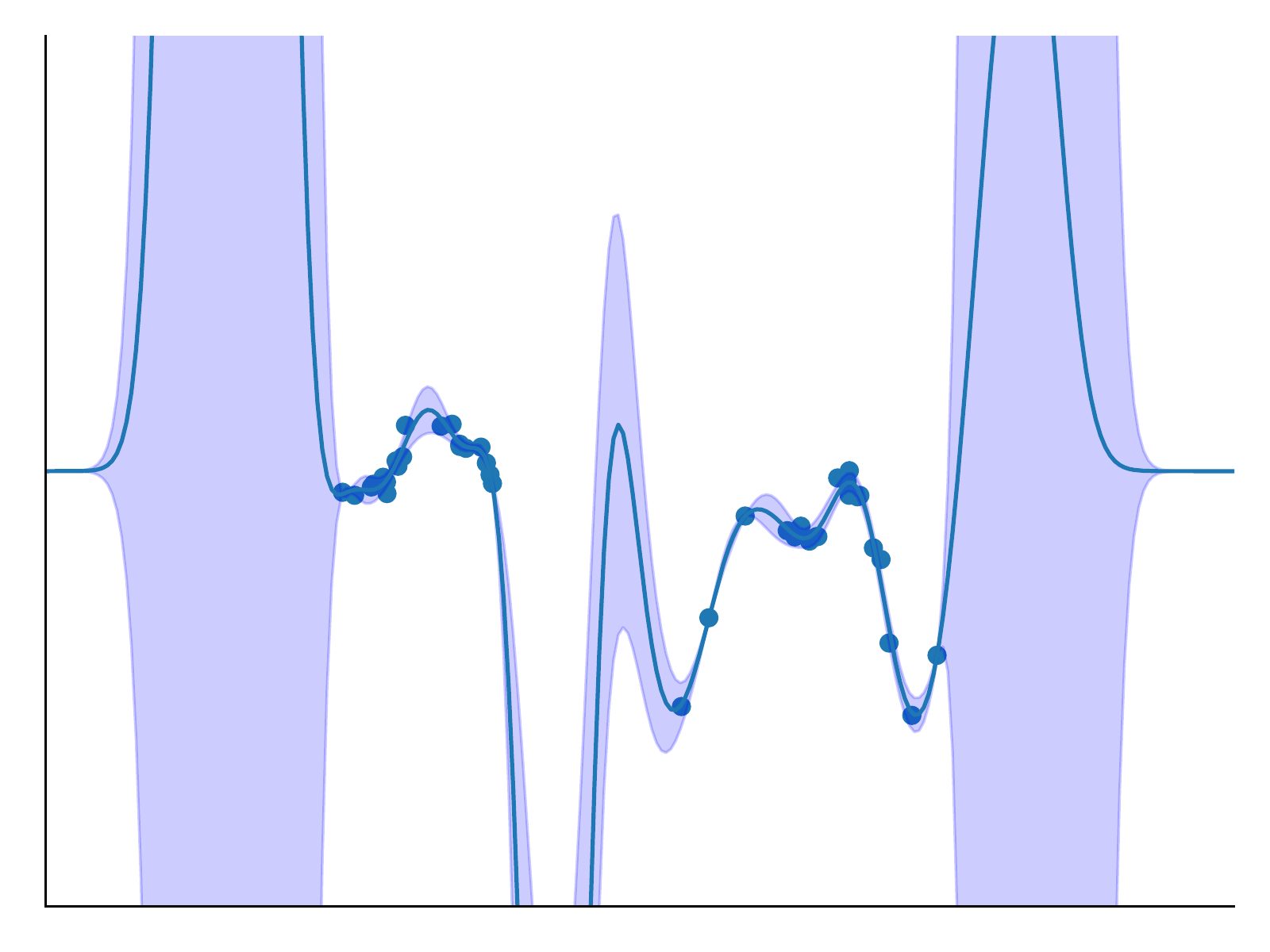}}
    \subfigure[\algname{RandA}]{
      \includegraphics[width=0.23\linewidth]{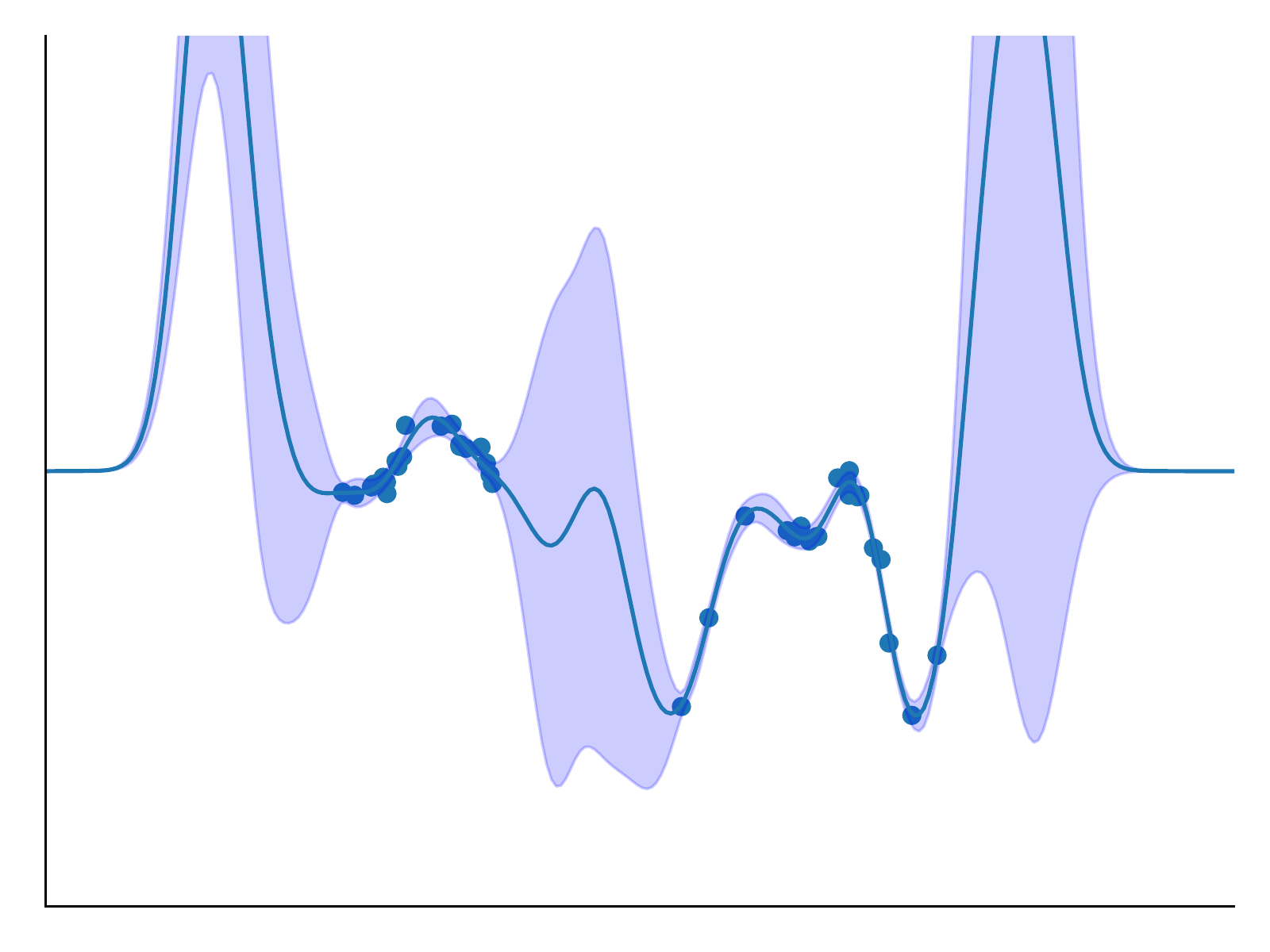}}
    \subfigure[\algname{SSGE}]{
      \includegraphics[width=0.23\linewidth]{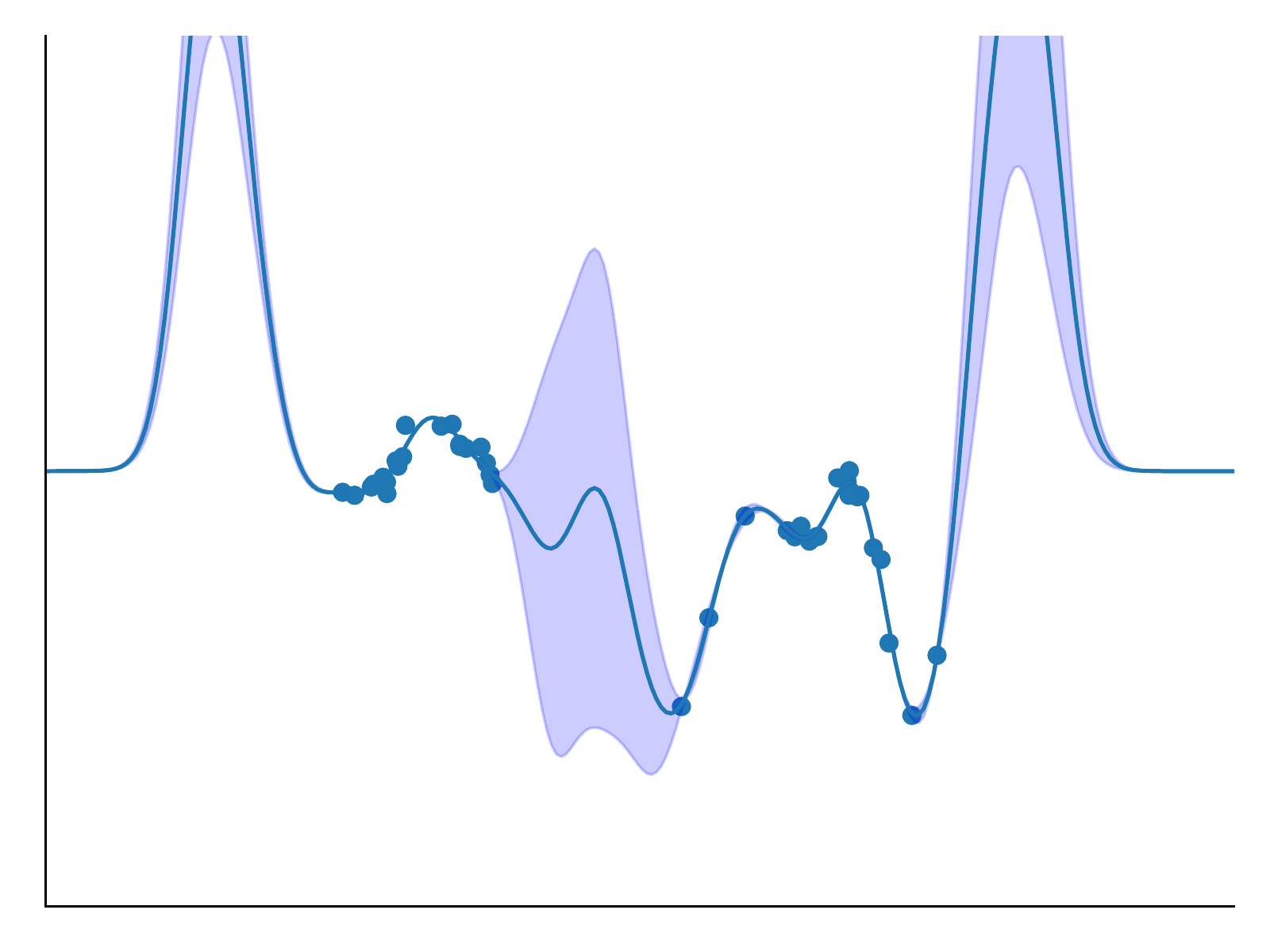}}
  }
\end{figure}

Following \citet{sun2018functional}, we consider the modified variational objective,
\begin{equation}\label{eqn:modified-elbo}
    \max_{Q \in \variationalfamily}\left( \mathbb{E}_{\Theta \sim Q}[\log p(y|x,\Theta)] - \mathbb{E}_{A \sim\mu}[\mathrm{D}_{KL}(Q_A, P_A)]\right)
\end{equation}
where $P_{A}, Q_A$ denote marginals of $P,Q$ at points indexed by $A$ and $\mu$ is a measure on subsets of $\mathcal{X}$. We note that \cref{eqn:modified-elbo} may be finite, even in cases where \cref{eqn:elbo} is not. We compare four algorithms using both full-covariance Gaussians (\algname{Full}) and fully-factorised Gaussians (\algname{FFG}) as the approximating families (noting that \algname{Full} will contain the true posterior). As a baseline, we consider using the \algname{Exact} KL divergence, i.e. $\mathrm{D}_{KL}(Q, P)$, which we can obtain since the KLs are the same in weight and function-space. In \algname{FixedA}, we randomly select a set of input points $A$ and keep it fixed throughout training in \cref{eqn:modified-elbo}. For \algname{RandA}, we use the approach proposed in \citet{sun2018functional} and sample a different set $A$ at each iteration, so that we Monte Carlo evaluate $\mathbb{E}_{A\sim\mu}[\mathrm{D}_{KL}(Q_A, P_A)]$ in \cref{eqn:modified-elbo}. However, as we can evaluate the KL exactly in this case, we do not use SSGE, instead leaving it for the final algorithm, \algname{SSGE}, which uses the random sampling scheme as well. Therefore, \algname{SSGE} is similar to the implementation in \citet{sun2018functional}, although we do not use their heuristic for re-scaling the KL term to reduce over-fitting (as they note that the modified objective will underestimate the KL term since it cannot achieve the supremum over all finite inputs). We consider two experiments, for which we provide additional experimental details in \cref{app:experimental-details}.

\paragraph{Toy experiment} We generate a synthetic 1D dataset by sampling a 20-dimensional weight vector from the prior and applying it to 20 radial basis function features. We then use this model to sample 40 noisy $(x, y)$ pairs with a noise standard deviation of $0.1$. We perform full-covariance inference with each method and plot the predictive posteriors in \Cref{fig:toy}. The posteriors found using the approximate KL divergence are prone to over-fitting; in the case of \algname{FixedA} it can be shown the optimal mean behaves like a combination of maximum likelihood estimation (\algname{MLE}) and maximum a posteriori (\algname{MAP}) inference (\cref{app:blr-elbo}).

\paragraph{UCI regression task} We fit a sparse variational GP \citep{titsias2009variational} with a squared-exponential kernel, with a separate length-scale for each input dimension. The learned kernel parameters and inducing points form a good representation of the training data \citep[\S 7]{wahba1990spline}. We use these features for linear regression. We show the negative log predictive densities (NLPDs) and KLs to the true posterior for the \dsetname{boston}, \dsetname{concrete}, and \dsetname{energy} datasets in \Cref{fig:uci-small}, providing additional results in \cref{app:experimental-details}.

\begin{figure}
    \centering
    \includegraphics[width=\linewidth]{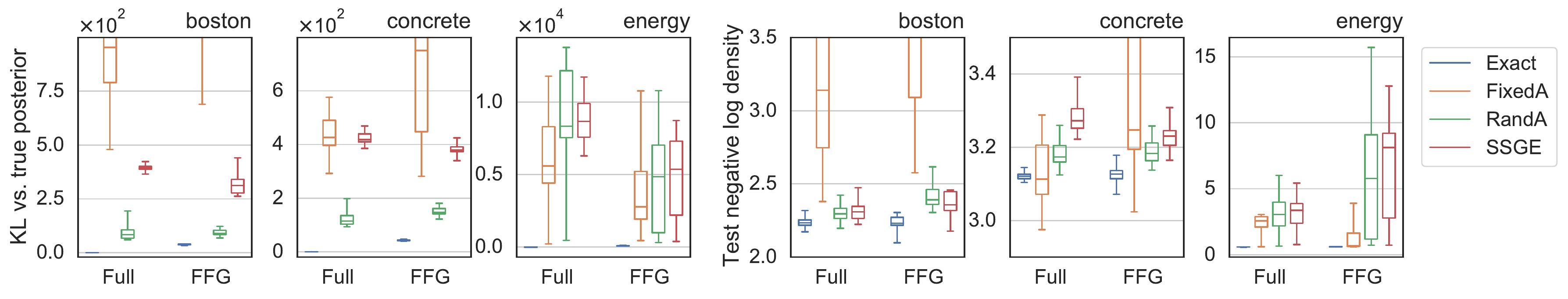}
    \caption{NLPD and KL for the first UCI data sets (alphabetic order). Lower is better. In \dsetname{boston} and
     \dsetname{concrete}, the KL divergence to the posterior of \algname{SSGE} is much larger than the one from \algname{RandA}.}
    \label{fig:uci-small}
\end{figure}
\paragraph{Conclusions}
In general, we observe from both experiments that, of the approximate methods \algname{RandA}, performs the best in terms of matching the true posterior. However, we note that all the approximate methods exhibit some amount of overfitting, which we would expect since the true functional KL divergence is obtained by taking the supremum over \emph{all} finite marginals. Finally, we note that \algname{SSGE} tends to perform worse than \algname{RandA}, which we would expect as it introduces an additional approximation. It is our hope that using this benchmark will help researchers find ways of improving these methods; for example, using other methods of implicit inference to reduce the discrepancy between \algname{RandA} and a version of inference that does not rely on Gaussianity.

\bibliography{main}


\appendix

\section{Measure-theoretic Definitions and Lemmas}
In this section, we recall several definitions and lemmas that will be useful in formalizing the results in the main text.  

We say two measures $P, Q$ on a common measurable space are \emph{equivalent} and write $P \sim Q$ if $P \ll Q$ and $Q \ll P$. We say $P$ and $Q$ are \emph{mutually singular} and write $P \perp Q$ if there exists a (measurable) event $E$ such that $P(E)=0$ and $Q(E^c)=0$ (where $E^c$ denotes the complement of $E$). In the case of probability measures, this is the same as $P(E)=0$ and $Q(E)=1$.

\begin{lem}\label{lem:singular-equiv}
Let $Q,\,Q'$, and $P$ be three measures on the measurable space $(A,\, \Sigma_A)$. Then $Q \sim Q'$ and $Q' \perp P$ implies that $Q \perp P$.
\end{lem}
\begin{proof}
By the assumption $Q' \perp P$, there exists an event $E \in \Sigma_A$ such that $Q'(E) = 0$ and $P(E^c) = 0$. Since $Q \sim Q'$, $Q'(E)=0$ implies $Q(E)=0$.
\end{proof}

\begin{defi}[Canonical projection]\label{def:canonical-pi}
For any $A\subset \mathcal{X}$, let $\pi_A: \R^\mathcal{X}\to \R^A$ denote the canonical projection onto $A$, i.e. $\pi_A(f) = (f(a))_{a\in A}$.
\end{defi}

\subsection{Product $\sigma$-algebra}
For any finite $A \subset \mathcal{X}$, we let $\lambda_A$ denote Lebesgue measure on $\R^A$, restricted to the Borel $\sigma$-algebra. As in \citet{sun2018functional}, we consider the product $\sigma$-algebra on $\R^{\mathcal{X}}$ i.e.~the coarsest $\sigma$-algebra on $\R^\mathcal{X}$ such that $\pi_{\{x\}}:\R^\mathcal{X} \to \R$ is measurable as a map to $\R$ equipped with the Borel $\sigma$-algebra for all $x \in \mathcal{X}$. For arbitrary $A \subset \mathcal{X}$ the map $\pi_A$ is measurable when $\R^A$ and $\R^\mathcal{X}$ are both equipped with their respective product $\sigma$-algebras (see \citet[Exercise 2.4.1.2]{tao2011introduction}).

\begin{lem}\label{lem:proj-singular}
Let $Q,P$ denote measures on $\Sigma$. Suppose there exists an $A \subset \mathcal{X}$ such that $\pi_{A*} Q \perp \pi_{A*} P$. Then $Q \perp P$. 
\end{lem}
\begin{proof}
Let $E \in \Sigma_A$ denote a witness to the orthogonality of $\pi_{A*} P \perp \pi_{A*} Q$. Then $\pi_{A}^{-1}(E) \in \Sigma$ by the measurability of projections. It follows from the definition of a pushforward measure that $\pi_{A}^{-1}(E)$ is a witness to the orthogonality of $P$ and $Q$.
\end{proof}

\subsection{Topological Lemmas}
As we work in the product $\sigma$-algebra generated by the Borel $\sigma$-algebra, we use several lemmas from topology in order to prove sets are measurable. We note that in general the product $\sigma$-algebra on $\R^{\mathcal{X}}$ generated by the Borel $\sigma$-algebra on $\R$ is not the same as the Borel $\sigma$-algebra generated by the product topology on $\R^\mathcal{X}$ when $\mathcal{X}$ is uncountable \citep[Exercise 2.4.1.5-6]{tao2011introduction}. 

\begin{lem}[Closed Mapping Lemma {\citep[Theorem 2.6]{conradfundamental}}]\label{lem:closed-map}
Suppose $X$ is compact and $Y$ is Hausdorff. Let $\phi: X \to Y$ continuous, then $\phi$ is a closed map.
\end{lem}
\begin{lem}\label{lem:countable-closed-map}
Suppose $X$ can be written as a countable union of compact subspaces and $Y$ is Hausdorff. Let $\phi: X \to Y$ continuous. Let $A \subset X$ be a closed set, the $\phi(A)$ is a countable union of closed sets.
\end{lem}
\begin{proof}
As $X$ is countable union of compact spaces, we can write $X = \bigcup_{i\in \mathbb{N}} X_i$, with $X_i$ compact. Define $A_i = A \cap X_i$, and note that $A_i$ is closed in the subspace topology of $X_i$. Define $\phi_i:X_i \to Y$ to be the restriction of $\phi$ to $X_i$; it follows from the definition of the subspace topology $\phi_i$ is continuous. Then,
\[
\phi(A) = \bigcup_{i \in \mathbb{N}} \phi_i(A_i).
\]
By the closed mapping lemma $\phi_i(A_i)$ is closed for all $i$.
\end{proof}

\section{$f$-divergences}\label{app:f-divs}

In this appendix, we briefly recall the definition of an $f$-divergence, as well as the necessary results to generalize our claims from the Kullback-Leibler divergence to other $f$-divergences. We use the definition from \citet{polyanskiy2014lecture},
\begin{defi}
Given a measurable space $(\Omega, \Sigma)$ and a convex function $f: [0,\infty) \to \R$ satisfying $f(1)=0$ which is strictly convex at $1$. For any two probability measures $P,Q$ on $\Sigma$, 
\begin{equation*}
    \mathrm{D}_f(Q,P) \coloneqq \int_{\{z:p(z)>0\}} f\left(\frac{q(z)}{p(z)}\right)p(z) d\mu(z) + f'(\infty)Q(\{z:p(z)=0\})
\end{equation*}
with $p(z)=\frac{dP}{d\mu}(z)$ and $q(z)=\frac{dQ}{d\mu}(z)$, $\mu$ is an arbitrary dominating measure (e.g.~$(P+Q)/2$), $f'(\infty) = \lim_{z\to 0^+} z f(1/z)$ and the understanding that if $Q(\{z:p(z)=0\})$ the second term is $0$ (even if $f'(\infty)=\infty$).
\end{defi}

Examples of $f$-divergences include KL divergence, total variation distance, squared Hellinger distance and $\alpha$-divergence. The data processing inequality holds for general $f$-divergences:

\begin{prop}[Data processing inequality~{\citet[Thm 6.2]{polyanskiy2014lecture}}]\label{prop:dpi-f}
Let $(A,\Sigma_A)$ and $(B, \Sigma_B)$ measurable spaces and $g: A \to B$ a $(\Sigma_A,\Sigma_B)$-measurable function, then for any $f$-divergence $D_f$, 
\begin{equation*}
\mathrm{D}_{f}\left(g_*P, g_*Q\right) \leq \mathrm{D}_{f}\left(P, Q\right),
\end{equation*}
where $g_*P$ indicates the pushforward measure of $P$ by $g$.
\end{prop}
 From this the analogue of \cref{prop:klinequality} holds for general $f$-divergences.
 
\begin{prop}\label{prop:fdiv-inequality}
Suppose the approximate posterior is parameterized by $(Q_\theta, g)$ and the prior is parameterized by $(P_\theta, g)$. Define $Q=g_*Q_\theta$ and $P=g_*P_\theta$ to be the approximate posterior predictive and the prior predictive respectively. Then for any $f$-divergence $D_f$,
\begin{equation}\label{eqn:dpi-fdiv-fvi}
    \mathrm{D}_{f}(Q, P) \leq \mathrm{D}_{f}(Q_\theta, P_\theta).
\end{equation}
Moreover, if $g$ is injective (each set of parameters corresponds to a unique predictive function) then equality holds in \cref{eqn:dpi-fvi}. 
\end{prop}
\begin{proof}
 For the inequality, using \cref{prop:dpi-f},
 \[
 \mathrm{D}_{f}(Q, P) =  \mathrm{D}_{f}(g_*Q_\Theta, g_*P_\Theta) \leq  \mathrm{D}_{f}(Q_\Theta, P_\Theta).
 \]
 Suppose $g$ is injective, then there exists a $g': \R^\mathcal{X} \to \R^k$ such that $g' \circ g(\theta)=\theta$ for all $\theta \in \R^k$ ($g$ has a left inverse). Then,
 \begin{align*}
 \mathrm{D}_{f}(Q, P) &= \mathrm{D}_{f}(g_*Q_\Theta, g_*P_\Theta) 
 \\ 
 &\geq \mathrm{D}_{f}(g'_*(g_*Q_\Theta), g'_*(g_*P_\Theta))\\
 &= \mathrm{D}_{f}((g' \circ g)_*Q_\Theta), (g' \circ g)_*P_\Theta)) \\
 &= \mathrm{D}_{f}(Q_\Theta, P_\Theta)). 
 \end{align*}
\end{proof}

If $P \perp Q$, then $Q(p=0)=1$, so that $\mathrm{D}_f(Q,P)=f'(\infty)$. Note that this value is the same for all $P\perp Q$ (and by the convexity of $f$ is the maximum value that can be obtained by the $f$-divergence), so that if all $Q \in \variationalfamily$ are mutually singular to $P$, then any $f$-divergence is entirely independent of which $Q \in \variationalfamily$ is selected. 

This leads to the generalizations of \cref{prop:bnn-diff-widths,prop:param-gp-kl} for other $f$-divergences:

\begin{prop}\label{prop:bnn-diff-widths-fdiv}
Suppose the approximate posterior is parameterized by $(Q_{\Theta_1},g_1)$ and the prior is parameterized by $(P_{\Theta_2},g_2)$, where both $Q_\theta$ and $P_\theta$ have densities (with respect to Lebesgue measure). Further suppose that $g_1(\Theta_1)$ is the mapping defined by a 1HL BNN with ReLU activation functions, $k$ neurons and parameters $\Theta_1$ and that $g_2(\Theta_2)$ is defined similarly, but with $j\neq k$ neurons. Let $Q=g_{1*}Q_{\Theta_1}$ and $P=g_{2*}P_{\Theta_2}$ denote the approximate posterior predictive and the prior predictive respectively, then
\[
\mathrm{D}_{KL}(Q,P)=f'(\infty).
\]
\end{prop}

\begin{prop}\label{prop:param-gp-fdiv}
Let $(Q_\Theta,g)$ parameterize the approximate posterio and $P$ be a (non-degenerate) Gaussian process. Assume that $g(\cdot)(x)$ is locally Lipschitz for all $x$. Let $Q=g_*Q_\theta$. Then, 
\begin{equation}
\mathrm{D}_f(Q,P) = f'(\infty) \text{\quad and \quad} \mathrm{D}_f(P,Q) = f'(\infty). 
\end{equation}
\end{prop}

By choosing $f(z)=z\log(z)$, we obtain the KL divergence used above. In this case $f'(\infty)=\lim_{z \to 0^+} \log(1/z) = \infty$.

\section{ReLU BNNs are mutually singular}\label{app:bnn-proof}

\begin{restatable}[]{prop}{bnns}\label{prop:bnn-diff-widths}
Suppose the approximate posterior is parameterized by $(Q_{\Theta_1},g_1)$ and the prior is parameterized by $(P_{\Theta_2}, g_2)$, where both $Q_{\Theta_1}$ and $P_{\Theta_2}$ have densities (with respect to the appropriate Lebesgue measures). Further suppose that $g_1(\Theta_1)$ is the mapping defined by a 1HL BNN with ReLU activation functions, $k$ neurons and parameters $\Theta_1$ and $g_2(\Theta_2)$ is defined similarly, but with $j \neq k$ neurons. Let $Q=g_{1*}Q_{\Theta_1}$ and $P=g_{2*}P_{\Theta_2}$ denote the approximate posterior predictive and the prior predictive respectively, then $\mathrm{D}_{KL}(Q,P)= \infty$.
\end{restatable}
\subsection{Preliminary Definitions and Lemmas}
We will construct a measurable event which one neural network assigns probability $1$ to and the other probability $0$. This event will roughly be functions $f \in \R^\mathcal{X}$ such that $f(x, 0, \dotsc, 0)$ is continuous with $k+1$ linear pieces. There are two steps: constructing an event that is measurable and captures this behavior; and showing that a ReLU network with a distribution over parameters with Lebesgue density and $k$ neurons assigns probability $1$ to this event. 

\begin{defi}\label{def:k-linear-pieces}
Let $f \in \R^\R$, we say \emph{$f$ is continuous with $k$-linear pieces} if there exists an $x_1\leq x_2 \leq \dotsc \leq x_{k-1}$ and $b, a_1, \dotsc, a_k$ such that 
\begin{align}
f(x) = a_{i+1}x + b_i  \textup{\quad for \,} x \in [x_i, x_{i+1})
\end{align}
with the understanding that $x_0= -\infty, x_k = \infty,\,b_0=b$, and where $b_i$ is selected so that the resulting function is continuous for $0<i\leq k$. 
\end{defi}

We say a function $\tilde{f} \in \R^\mathbb{Q}$ is \emph{continuous with $k$-linear pieces} if it can be extended to a function $f \in \R^\R$ that is continuous with $k$-linear pieces.

\begin{prop}
Define $E_k:=\{f \in \R^\R: \pi_{\mathbb{Q}}(f) \text{\, is continuous with $k$-linear pieces}\}$. Then $E_k$ is measurable in the product $\sigma$-algebra on $\R^\R$ induced by the Borel $\sigma$-algebra.
\end{prop}

\begin{proof}
By the measurability of $\pi_\rationals$, it suffices to show that 
\[\tilde{E}_k:= \{\tilde{f} \in \R^\rationals: \tilde{f} \text{\, is continuous with $k$-linear pieces}\} \] is measurable in the product $\sigma$-algebra on $\R^\rationals$.

We will use the previous lemma to show that  $\tilde{E}_k$ is a countable union of closed sets and is hence measurable in the Borel $\sigma$-algebra induced by the product topology on $\R^\rationals$. As this coincides with the product $\sigma$-algebra for countable products of $\R$ \citep[Exercise 2.4.1.5]{tao2011introduction}, this suffices. 

We will do this by showing that $\tilde{E}_k$ is the image of $\R^{2k+1}$ under a continuous function, and applying \cref{lem:countable-closed-map}. As $\R^{2k+1}$ is closed, $\R^\rationals$ is a product of Hausdorff spaces, hence Hausdorff and $\R^{2k+1} = \bigcup_{i \in \naturals} [-i,i]^{2k+1}$ (i.e~it is a countable union of compact set), all that remains is to construct a continuous $\phi: \R^{2k+1} \to \R^\rationals$ such that $\phi(\R^{2k+1})= \tilde{E}_k$.

Let 
\begin{align}
    \phi(\tilde{s}_0,\tilde{s}_1, \dotsc, \tilde{s}_{2k}) = \tilde{f}^{(\tilde{s}_0,\tilde{s}_1, \dotsc, \tilde{s}_{2k})}
\end{align}
with $\tilde{f}$ defined as in \cref{def:k-linear-pieces} (restricted to $\rationals$) with $b=\tilde{s}_0, a_i = \tilde{s}_i$ and $x_1=\tilde{s}_{k+1}$ and $x_{i+1} = x_i + |\tilde{s}_{i+k+1}|$. From this definition, it is clear that $\phi(\R^{2k+1}) \subset \tilde{E}_k$. The reverse inclusion follows from the noting that any function of the form in \cref{def:k-linear-pieces} can be written in this form. It remains to show $\phi$ is continuous. By the universal property of the product topology \citep[Theorem 19.6]{munkres2000topology}, we need only show $\phi_{\{q\}}$ is continuous for all $q \in \rationals$, which can be shown from the metric space definition of continuity (with some care for cases when $q$ is on the boundary of two linear regions).
\end{proof}

\subsection{Proof in the case when input space is one-dimensional}

We first prove \cref{prop:bnn-diff-widths} under the assumption that $\mathcal{X}=\R$; the generalization to multidimensional inputs is straightforward. 

\begin{proof}
All that remains to show that the implied measures for two 1HL ReLU BNNs mapping from $\R \to \R$ are orthogonal is showing that if a 1HL ReLU BNN has $k$ neurons, then it produces a function in $\tilde{E}_{k+1}\setminus\tilde{E}_k$ with probability 1.

We first show that with probability $1$, the implied function is in $E_{k+1}$ (in fact, this holds surely). Let $w^1, w^2, b^1 \in \R^{k}$ and $b^2 \in \R$ be an arbitrary realization of weights and biases, then $f(x) =b^{(2)} + \langle w^{(2)}, \max(0, w^{(1)} \circ x + b^{(1)})\rangle= b^{(2)}$, where $\circ$ denotes an element-wise vector product. We can rewrite this as
\[
f(x) = b^{(2)} + \sum_{\substack{i=1\\ w_i^{(1)} x +b^{(1)}_i}>0}^k w^{(1)}_iw^{(2)}_ix+ w_2b_1. 
\]
Note that this is piecewise linear, with boundaries at $\tilde{x}_i=-\frac{b_i}{w_i}$ for $i \leq k$, and is continuous in $x$ as it can be written as a composition of continuous functions. As this holds for arbitrary realizations of parameters, $f$ is surely in $E_{k+1}$. 

On the other hand, for $f$ to be in $E_{k}$ it must be the case that either:
\begin{itemize}
    \item $\tilde{x}_i = \tilde{x}_j$ for some $i \neq j$ (i.e. boundaries coincide).
    \item There exists an $i$ such that $w^{(1)}_iw^{(2)}_i=0$ (adjacent regions have the same slope). 
\end{itemize}
The above conditions define a Lebesgue null set; hence under the assumption that the distribution over parameters has density with respect to Lebesgue measure, this is a probability $0$ event.
\end{proof}
\subsection{Extension to multidimensional inputs}
The extension to multidimensional inputs is almost immediate up considering the set $E'_k:=\{f \in \R^{\R^d}: f(\cdot, 0, 0, \dotsc, 0) \text{\, is continuous with\, } $k$ \text{\, linear pieces}\}$. Defining $\tilde{E}'_k:= \{f \in \R^\rationals \times \R^{\R^{d-1}}: f(\cdot, 0, 0, \dotsc, 0) \text{\, is continuous with\, } k \text{\, linear pieces on rational\, } x\}$, we see that $\tilde{E}'_k$ is measurable for the same reason $\tilde{E}_k$ is measurable. Moreover, viewed along this slice of input space, the neural network is identical to a 1HL network mapping from $\R \to \R$, so the proof in the previous subsection holds without modification.

\section{Parametric models and non-degenerate Gaussian measures are mutually singular}\label{app:param-gp}

\subsection{Preliminaries}

The main ingredient in the result is the following lemma:

\begin{lem}[{\citet[Lemma 7.25]{rudin1966real}}]\label{lem:local-lipshitz}
Let $E \subset \R^k$ a Lebesgue null set. Suppose $f: \R^k \to \R^k$ satisfies, for all $x \in E$ there exists a $\delta>0$ and $M>0$ such that
\[
\frac{\|f(x)-f(y)\|}{\|x-y\|} \leq M
\]
for all $y \in E \cap B(x; \delta)$ where $B(x; \delta)$ denotes the ball of radius $\delta$ centered at $x$. Then $\lambda_{k}(f(E))=0$.
\end{lem}
Note that the condition on $f$ holds if $f$ is locally Lipschitz. Moreover, if a function from $\R^k \to \R^k$ is locally Lipschitz for every output index, it is locally Lipschitz as a function from $\R^k \to \R^k$.

\subsection{Statement and Proof}
\parGPs*

\begin{proof}
Fix a set $A \subset \mathcal{X}, |A|=k+1$ so that $\pi_{A*}P \sim \lambda_{A}$. Such an $A$ exists by the assumption that $P$ is non-degenerate and $\mathcal{X}$ is infinite. By \cref{lem:singular-equiv} and \cref{lem:proj-singular} it then suffices to show $\pi_A Q \perp \lambda_{A}$. 

Define the event $E=(\pi_A \circ g)(\R^k)$. As $\pi_A \circ g$ is continuous, by the closed mapping lemma, $E$ is a countable union of closed sets (\cref{lem:countable-closed-map}), hence Borel measurable. 
Also,
\begin{align*}
    \pi_{A_*}Q(E) = Q_\theta(g^{-1}(\pi_A^{-1}(\pi_A(g(\R^k))))) \geq Q_\theta(g^{-1}(g(\R^k))) \geq Q_\theta(\R^k)=1.
\end{align*}
The inequalities follow from the pre-image of the image of a set containing the original set. 

All that is left to show is that $E$ is a Lebesgue null set. Let $\tilde{g}_A:\R^{k+1} \to \R^A$ be defined by $\tilde{g}_A(x_1,\dots, x_{k+1})=(\pi_A \circ g)(x_1, \dots, x_k)$. Then $ \tilde{g}_A(\R^{k} \times \{0\})=E$. As $\pi_A \circ g$ is locally Lipschitz, so is $\tilde{g}_A$. The proof is then completed by \cref{lem:local-lipshitz}, noting that $\R^{k} \times \{0\}$ is a Lebesgue null set in $\R^{k+1}$.
\end{proof}

\section{Bayesian Linear Regression and Variational Inference}\label{app:blr-elbo}

Consider the Bayesian linear regression model \cref{eqn:blr} and the modified variational objective,
\[
\max_{Q \in \variationalfamily} \mathbb{E}_{\Theta \sim Q}[\log p(y|x,\Theta)] - \KL(Q_A, P_A), \]
which is a special case of \cref{eqn:modified-elbo} when $\mu$ is a point mass on $A$ (i.e.~\algname{FixedA}). Let $\datafeatures$ be the $n \times k$ feature matrix with $[\datafeatures]_{ij} = \phi_j(x_i)$. Let $A=\{a_1,\cdots, a_m\}$ and $\Afeatures$ be the $m \times k$ feature matrix with $[\Afeatures]_{ij} = \phi_j(a_i)$. Without loss of generality, we assume that $\Afeatures$ has linearly independent rows, so that $Q_A,P_A$ are non-degenerate. Then,
\[
\log p(y|x, \Theta) = -\frac{n}{2} \log (2\pi \sigma^2)^{-n/2} - \frac{1}{2 \sigma^2}(y - \datafeatures \Theta)\transpose (y - \datafeatures W)
\]
For $Q = \mathcal{N}(\mu_Q, \Sigma_Q)$, we have
\begin{align*}
    \mathbb{E}_{W\sim Q}[\log p(y|W)] &=  -\frac{n}{2} \log (2\pi \sigma^2) - \frac{1}{2\sigma^2}\mathbb{E}_{W\sim Q}[(y - \datafeatures W)\transpose (y - \datafeatures W)] \\
    & = -\frac{n}{2} \log (2\pi \sigma^2) - \frac{1}{2\sigma^2}\left(y\transpose y - 2y\transpose \datafeatures \mathbb{E}[W] +\mathbb{E}[W\transpose \datafeatures\transpose\datafeatures W] \right) \\ 
    & = -\frac{n}{2} \log (2\pi \sigma^2) - \frac{1}{2\sigma^2}\left(y\transpose y - 2y\transpose \datafeatures \mu_Q +\mathbb{E}[W\transpose \datafeatures\transpose\datafeatures W] \right) \\
    & = -\frac{n}{2} \log (2\pi \sigma^2) - \frac{1}{2\sigma^2}\left(y\transpose y - 2y\transpose \datafeatures \mu_Q +\Tr(\datafeatures\transpose\datafeatures \Sigma_Q) + \mu_Q\transpose\datafeatures\transpose\datafeatures \mu_Q \right) \\
    & = -\frac{n}{2} \log (2\pi \sigma^2) - \frac{1}{2\sigma^2}\left((y- \datafeatures \mu_Q)\transpose(y- \datafeatures \mu_Q) +\Tr(\datafeatures\transpose\datafeatures \Sigma_Q) \right).
\end{align*} 
The gradient of this term with respect to both $\mu_Q$ is,
\[
\nabla_{\mu_Q} \mathbb{E}_{W\sim Q}[\log p(y|W)] = \frac{1}{\sigma^2}\datafeatures\transpose (y- \datafeatures \mu_Q).
\]

 We now turn out attention to the KL divergence between $Q_A$ and $P_A$. We assume $P= \mcN(0, I)$. We then have $Q_A = \mcN(\Afeatures \mu_Q, \Afeatures \Sigma_Q \Afeatures \transpose)$ and  $P_A = \mcN(0, \Afeatures \Afeatures \transpose)$. The KL divergence is then,
 \begin{align*}
     \KL(Q_A,P_A) = \frac{1}{2}\Bigg(-m + \mu_Q\transpose\Afeatures\transpose(\Afeatures\Afeatures \transpose)^{-1}&\Afeatures \mu_Q  + \Tr( (\Afeatures \Afeatures \transpose)^{-1} \Afeatures \Sigma_Q \Afeatures \transpose) \\ &-  \log \Big\lvert(\Afeatures\Afeatures \transpose)^{-1} \Afeatures \Sigma_Q \Afeatures \transpose\Big\rvert \Bigg),
 \end{align*}
 and its gradient with respect to $\mu_Q$ is, 
 \[
 \nabla_{\mu_Q} \KL(Q_A,P_A) = \Afeatures\transpose(\Afeatures\Afeatures \transpose)^{-1}\Afeatures \mu_Q.
 \]
Note that $\Afeatures\transpose(\Afeatures \Afeatures \transpose)^{-1}\Afeatures$ is the projection of $\mu_Q$ onto the column space of $\Afeatures$.

We can find the optimal $\mu_Q$ solutions to our optimization problem by setting gradient equal to $0$. This yields,
\[
 \frac{1}{\sigma^2}\datafeatures\transpose y-  \frac{1}{\sigma^2}\datafeatures\transpose\datafeatures \mu_Q - \Afeatures\transpose(\Afeatures \Afeatures \transpose)^{-1}\Afeatures\mu_Q = 0.
\]
Rearranging,
\[
\left(\frac{1}{\sigma^2}\datafeatures\transpose\datafeatures + \Afeatures\transpose(\Afeatures \Afeatures \transpose)^{-1}\Afeatures\right) \mu_Q =  \frac{1}{\sigma^2}\datafeatures\transpose y . 
\]
A solution for $\mu_Q$ (though this solution is not in general unique) is given by,
\[
\mu_Q =  (\datafeatures\transpose\datafeatures + \sigma^2 \Afeatures\transpose(\Afeatures \Afeatures \transpose)^{-1}\Afeatures)^{\dagger}\datafeatures\transpose y.
\]
where we use $\dagger$ to denote the pseudo-inverse.

The maximum likelihood solution could be found by removing the term $\Afeatures\transpose(\Afeatures \Afeatures \transpose)^{-1}\Afeatures$, while map inference is recovered when $\Afeatures\transpose(\Afeatures \Afeatures \transpose)^{-1}\Afeatures=I$. As $\Afeatures\transpose(\Afeatures \Afeatures \transpose)^{-1}\Afeatures$ is a projection matrix, we see this has the intuitive interpretation of regularizing the MLE solution, but only in the directions that effect predictions on the points in $A$.

\section{Experimental Details}\label{app:experimental-details}
In this section, we present details on the experiments we performed.

\subsection{Toy experiment}
We generate the data using 20 linearly-spaced radial basis function features between -2 and 2, with each feature having lengthscale 0.2. The input data are sampled by taking 20 samples from $\mathcal{N}(-1.2, 0.3^2)$ and 20 from $\mathcal{N}(1.2, 0.3^2)$. We sample the weights from a standard normal prior and generate the output data as described earlier.

For each method, we optimize the ELBO using Adam \citep{kingma2014adam} using 5000 gradient steps with a learning rate of 0.01. We evaluate the likelihood term using the full dataset every iteration without minibatching. For the \algname{FixedA}, \algname{RandA}, and \algname{SSGE} methods we use 10 points for $A$, with 5 from selected from the dataset, and the other 5 sampled uniformly from a box with bounds determined by the bounds of the training data.

\subsection{UCI experiments}
We begin by training sparse Gaussian process regression (SGPR) \citep{titsias2009variational} for each split of the dataset we use, using 100 inducing points initialized by sampling from the dataset. We use the GPflow \citep{GPflow2017} implementation of SGPR, initializing the ARD lengthscales to 1 and the noise variance to 1e-4. We use the built-in Scipy optimizer with a maximum of 5000 iterations. After the SGPR model is trained, we use the 100 inducing points to obtain 100 features to use for Bayesian linear regression \citep{wahba1990spline}, with a standard normal prior over the weights. We then train each of the methods for estimating the KL term, keeping the learned noise variance from the SGPR models fixed.

For \algname{Exact}, we simply use the closed-form solutions to the inference problem, which exist for both \algname{Full} and \algname{FFG}. For the algorithms that make use of \cref{eqn:modified-elbo}, we set $|A|=80$ and choose half of its points from the training set (note this differs slightly from \citet{sun2018functional}, who always include the current mini-batch used to evaluate the expected log likelihood term in $A$), and the other half uniformly randomly over the bounding box of the training set. For \algname{FixedA}, we optimize using L-BFGS \citep{nocedal1980updating}, whereas for \algname{RandA} and \algname{SSGE} we optimize using Adam. In each case we optimize for (up to) 15000 steps. We note that we add jitter as necessary to ensure that the Cholesky decompositions we perform succeed. Finally, we average each run over 20 seeds.

Extended results are displayed in figures~\ref{fig:uci-composite-bcek}~and~\ref{fig:uci-composite-npwy}.

\begin{figure}[p]
\floatconts
  {fig:uci-composite-bcek}
  {\caption{Results for datasets \dsetname{boston}, \dsetname{concrete}, \dsetname{energy} and \dsetname{kin8nm}. The different colors represent different sizes of $|A|$ and the minibatch. From left to right: 20, 40, 80, 160.}}
  {%
      \includegraphics[width=\textwidth]{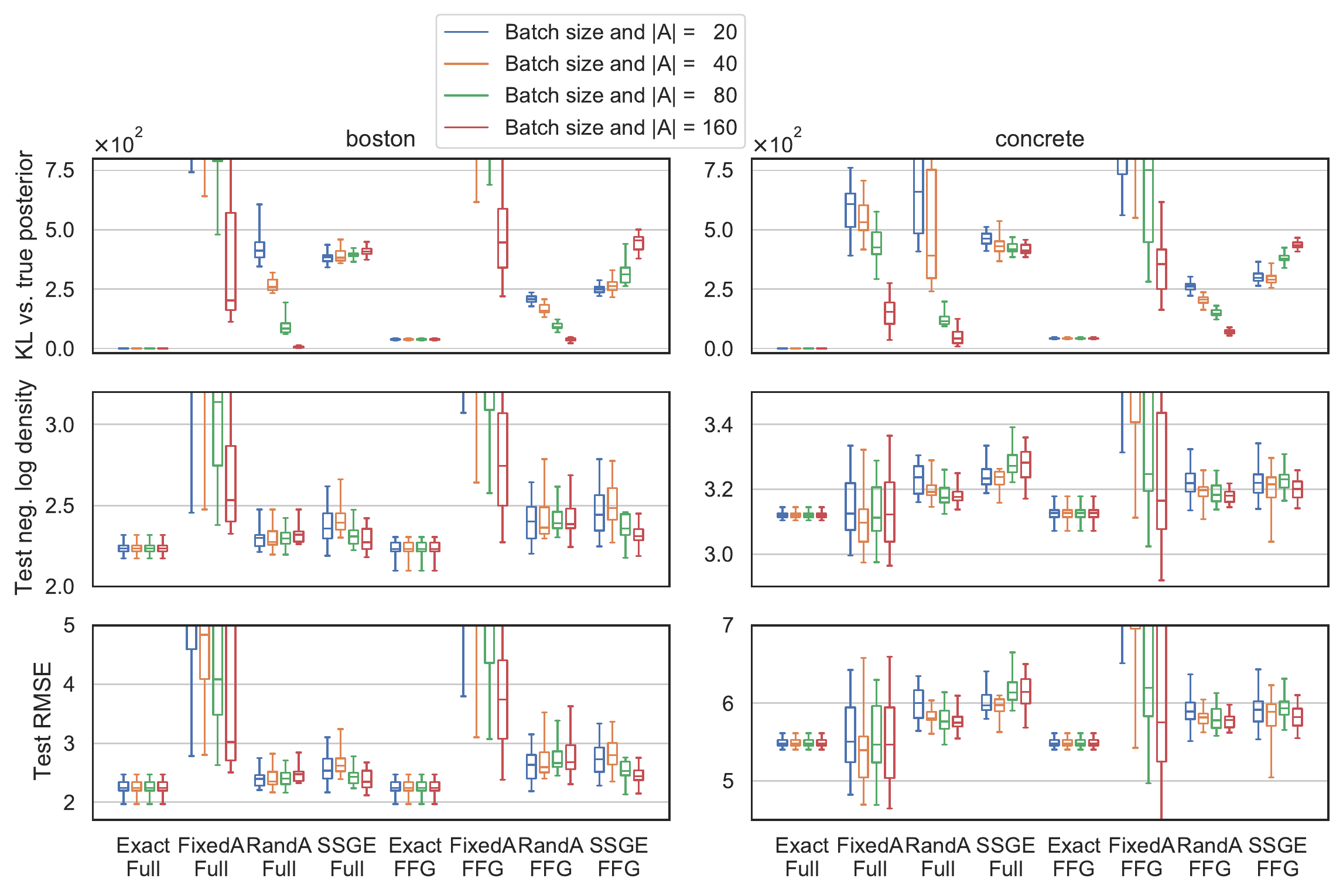}
      \includegraphics[width=\textwidth]{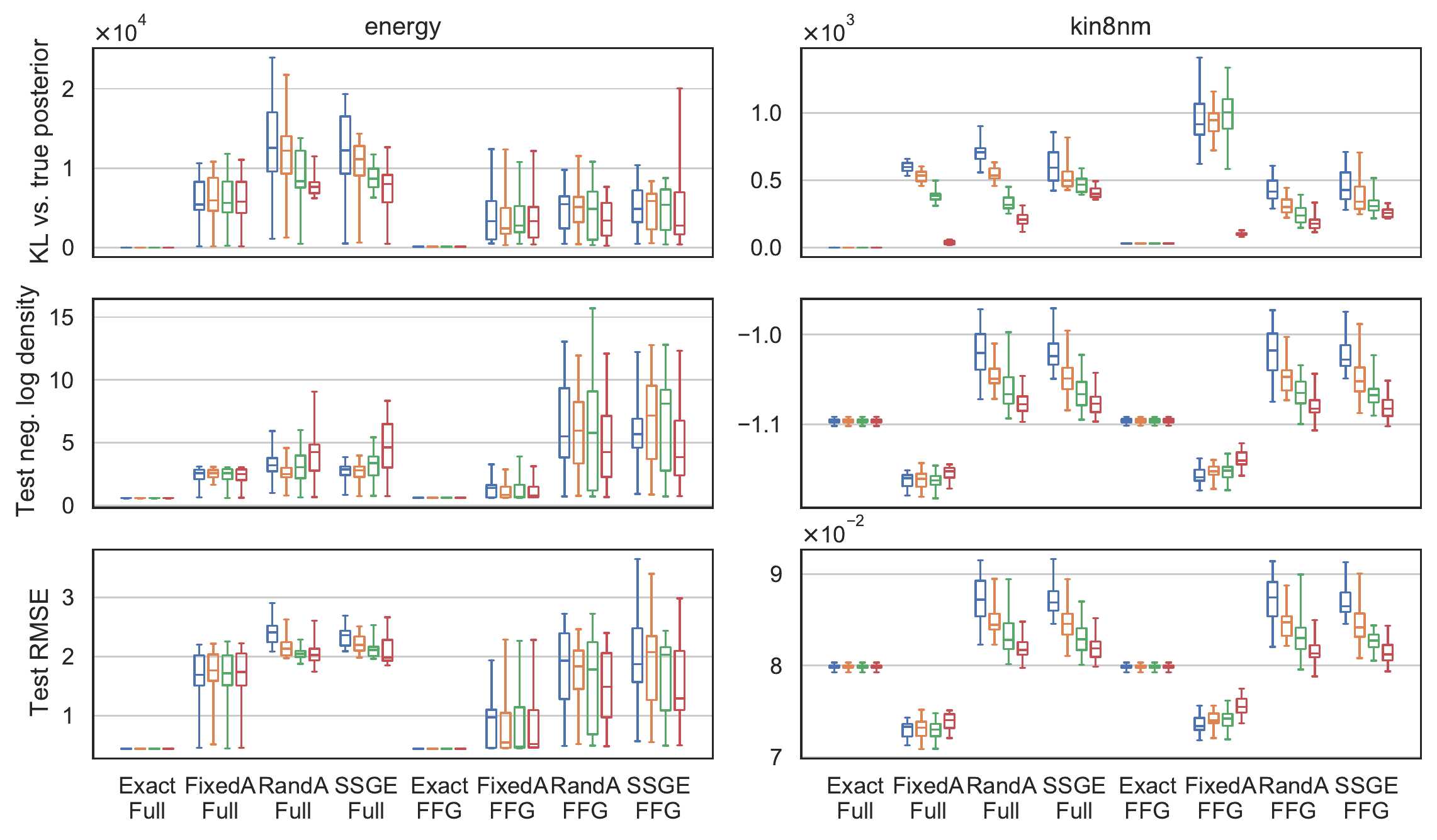}
  }
\end{figure}
\begin{figure}[p]
\floatconts
  {fig:uci-composite-npwy}
  {\caption{Results for datasets \dsetname{naval}, \dsetname{power}, \dsetname{wine} and \dsetname{yacht}. The different colors represent different sizes of $|A|$ and the minibatch. From left to right: 20, 40, 80, 160.}}
  {%
      \includegraphics[width=\textwidth]{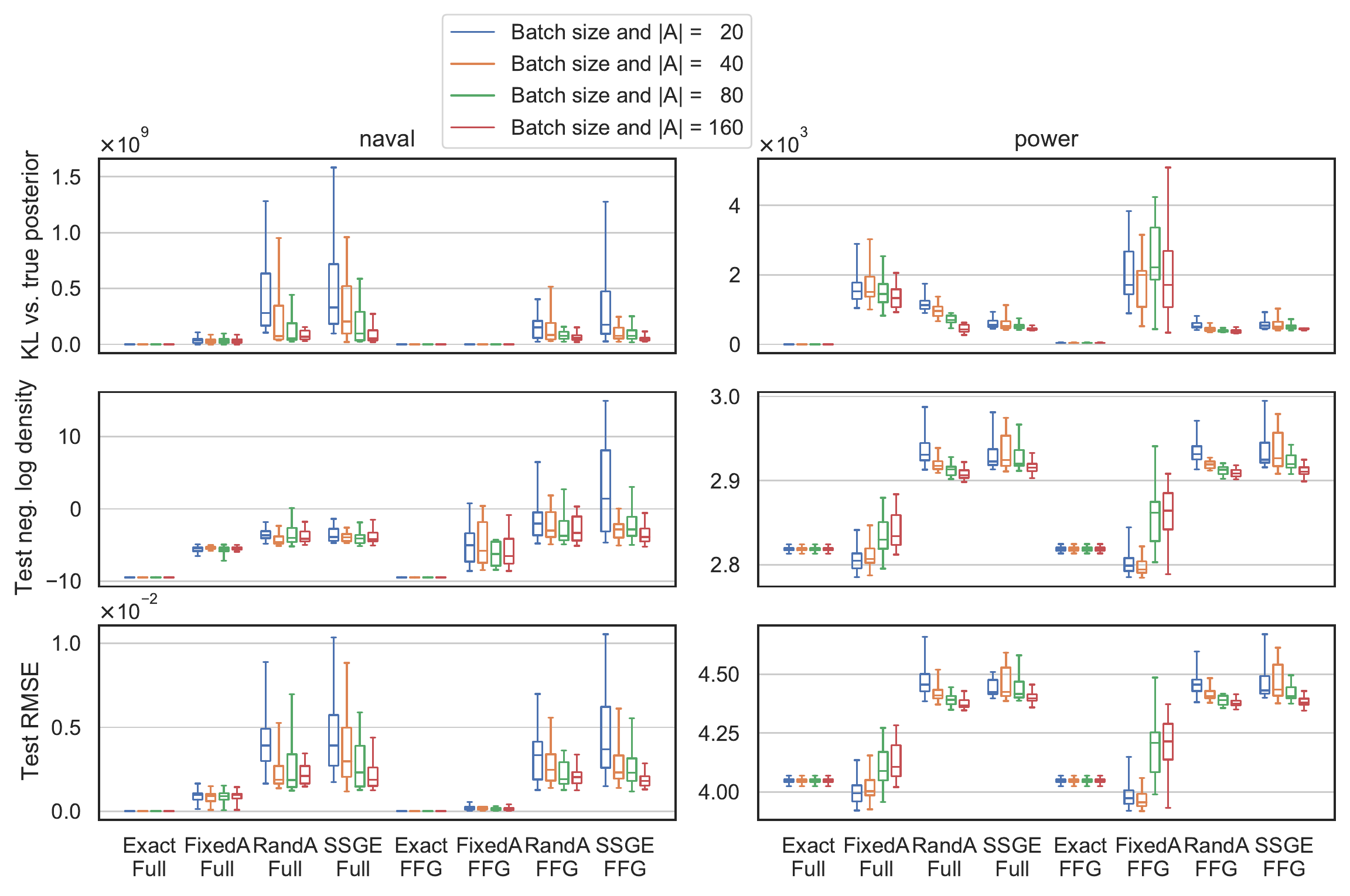}
      \includegraphics[width=\textwidth]{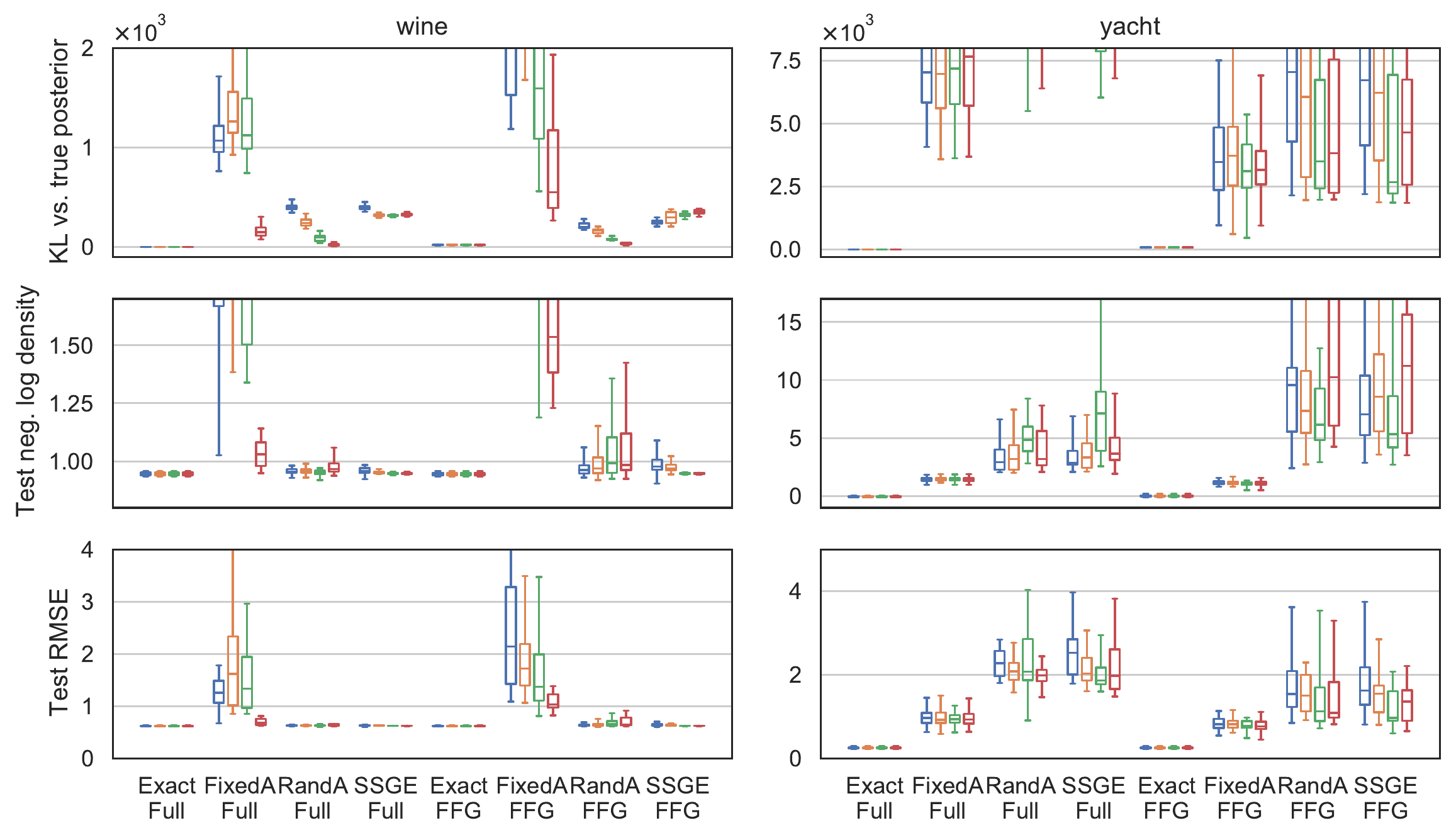}
  }
\end{figure}

\end{document}